\documentclass{article}
\usepackage{PRIMEarxiv}

\usepackage[utf8]{inputenc} 
\usepackage[T1]{fontenc}    
\usepackage{hyperref}       
\usepackage{url}            
\usepackage{booktabs}       
\usepackage{amsfonts}       
\usepackage{nicefrac}       
\usepackage{microtype}      
\usepackage{xcolor}         

\usepackage{amsmath}
\usepackage{float}
\usepackage{wrapfig}
\usepackage{graphicx}
\usepackage{subfigure}
\usepackage{subcaption}
\usepackage{multirow}

\usepackage{listings}
\usepackage{amsfonts}
\usepackage{gensymb}
\usepackage{blindtext}
\usepackage{mathtools}
\usepackage{amsthm}
\usepackage{bm}

\newtheorem{theorem}{Theorem}[section]

\newtheorem{corollary}{Corollary}[theorem]
\newtheorem{defn}[theorem]{Definition}

\usepackage[toc,page]{appendix}

\title{Full-range Head Pose Geometric Data Augmentations}

\author{
Hu, Huei-Chung\\
  Docomo Innovations \\
  \texttt{heidi.hu@docomoinnovations.com} \\
    \and
Wu, Xuyang  \\
    Santa Clara University \\
    \texttt{xwu5@scu.edu}
    \and
Liu, Haowei \\
    Santa Clara University \\
    \texttt{hliu6@scu.edu}
    \and
Wei, Ting-Ruen  \\ 
    Santa Clara University \\
    \texttt{twei2@scu.edu}
    \and
Wu, Hsin-Tai\footnotemark[1]\\ 
    Docomo Innovations \\
    \texttt{hwu@docomoinnovations.com}
}

\begin{document}

\maketitle

\begin{abstract}
Many head pose estimation (HPE) methods promise the ability to create full-range datasets, theoretically allowing the estimation of the rotation and positioning of the head from various angles. However, these methods are only accurate within a range of head angles; exceeding this specific range led to significant inaccuracies. This is dominantly explained by unclear specificity of the coordinate systems and Euler Angles used in the foundational rotation matrix calculations. Here, we addressed these limitations by presenting (1) methods that accurately infer the correct coordinate system and Euler angles in the correct axis-sequence, (2) novel formulae for 2D geometric augmentations of the rotation matrices under the (SPECIFIC) coordinate system, (3) derivations for the correct drawing routines for rotation matrices and poses, and (4) mathematical experimentation and verification that allow proper pitch-yaw coverage for full-range head pose dataset generation. Performing our augmentation techniques to existing head pose estimation methods demonstrated a significant improvement to the model performance. Code will be released upon paper acceptance.
\end{abstract}

\begin{figure}[H]
\centering
\includegraphics[width=0.64 \textwidth]{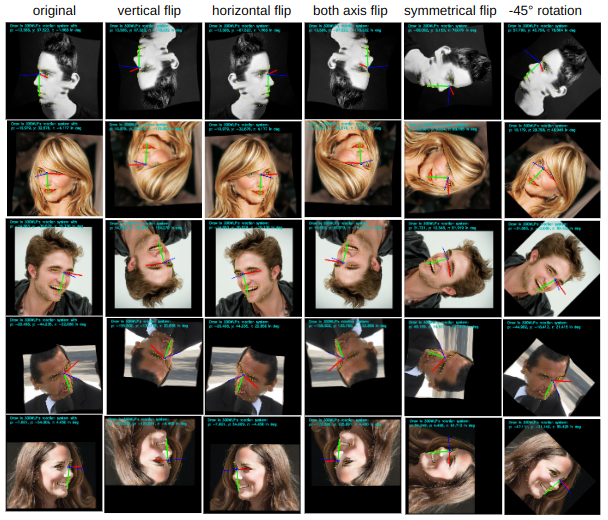}
\caption{Illustration of our 3D flipping and rotation augmentation formula (Corollary \ref{cor:300wlp_flipcor}) for the 300W-LP dataset.}
\label{fig:pose_augmentation}
\end{figure} 

\section{Introduction}
Head Pose Estimation (HPE) has evolved significantly, driven by the increasing integration of computer vision and vision-based deep learning \cite{DBLP:journals/sncs/AspertiF23}. Early approaches often relied on facial key-point detection and images with fixed rotation angle labeling. However, this field witnessed a paradigm shift with the advent of deep learning, especially Convolutional Neural Networks (CNNs). These models could automatically learn hierarchical features from data, thus notably improving the accuracy and robustness of HPE algorithms.

However, it is noteworthy that many studies on HPE overlook fundamental definitions of the coordinate systems employed. We observed that orientation axis-drawing routines are often borrowed from code intended for unknown or unspecific coordinate systems. We address these fundamental problems and derive the math for geometric rotation augmentations. In addition, we showcase how these changes improve the existing model performance and streamline the process of dataset creation.

\begin{figure}[H]
\centering
\includegraphics[width=1.0 \textwidth]{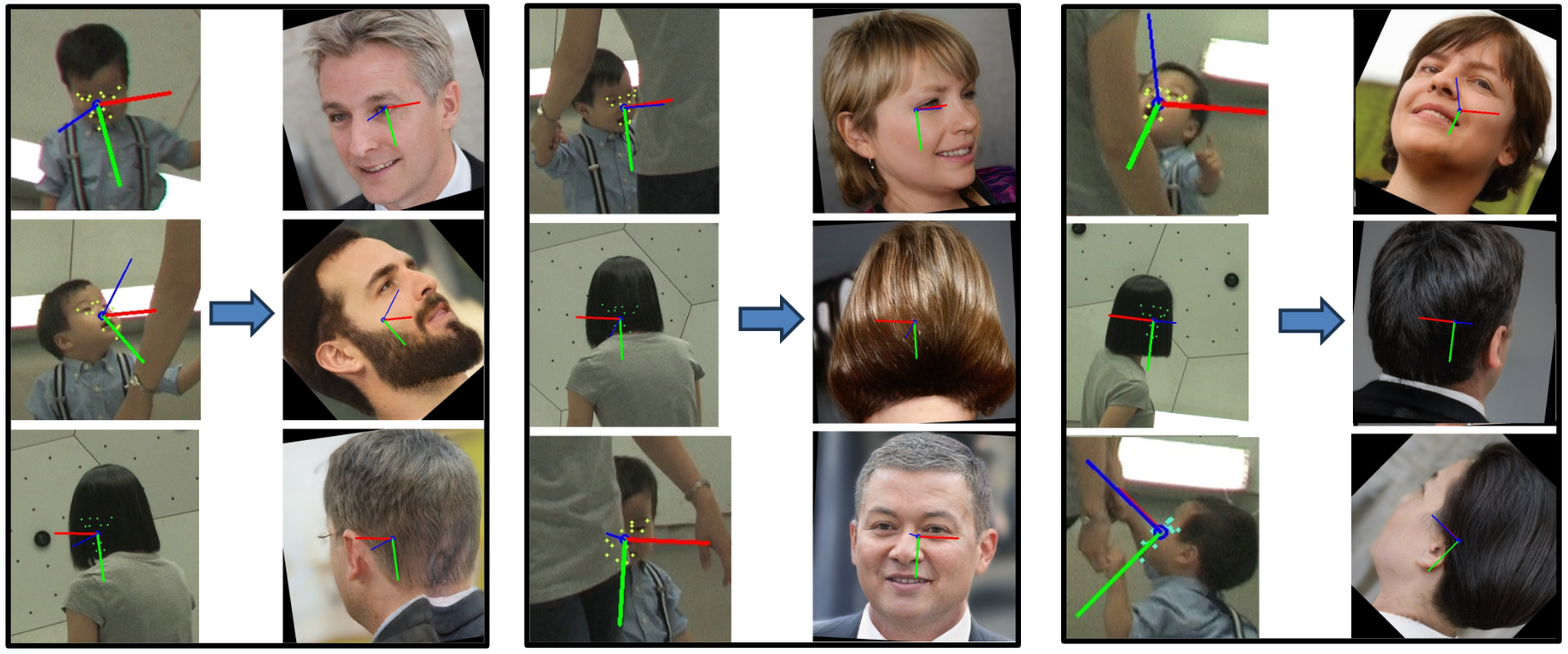}
\caption{Applying Fig. \ref{listing:1} with Panohead: we extract head rotations from CMU\_HPE\_10K, transform them into $(p', y', r')$ triples with Fig. \ref{listing:1}, use Panohead to create synthetic human head images of the given $p', y'$, and finally apply the roll rotations of angles $r'$.}
\label{fig:pano_cmu1}
\end{figure} 

\section{Related Works}
This paper will focus on monocular RGB images. Classical HPE approaches preceding 2008 can be found in the survey paper \cite{hpe_article}.  Deformable models are thoroughly discussed in \cite{1227983, DBLP:conf/cvpr/ZhuLLSL16, guo2020towards, zhu2017face}, and were used to create the commonly used synthetic pose datasets such as 300W-LP \cite{DBLP:conf/cvpr/ZhuLLSL16}.

\subsection{Head Pose Dataset Creation}\label{sec:hpe_dataset_creation}
The 300W across Large Poses (300W-LP) dataset \cite{DBLP:conf/cvpr/ZhuLLSL16} was based on 300W \cite{DBLP:conf/iccvw/SagonasTZP13}, which standardized multiple alignment datasets with 68 landmarks, including AFW \cite{DBLP:conf/cvpr/ZhuR12}, LFPW \cite{DBLP:conf/cvpr/BelhumeurJKK11}, HELEN \cite{DBLP:conf/iccvw/ZhouFCJY13}, IBUG \cite{DBLP:conf/iccvw/SagonasTZP13}, and XM2VTS \cite{Messer1999XM2VTSDBTE}. Its Euler-angle labels are estimated from \cite{DBLP:conf/cvpr/ZhuLLSL16}'s 3D Dense Face Alignment (3DDFA), which applies Blanz et al.'s morphable model 3DMM \cite{1227983}, to obtain the standard and rotated 3D faces tailored for an image. It doubled its dataset with the vertical flipping augmentation. AFLW2000-3D (AFLW2000) \cite{zhu2017face}, similar to 300W-LP's construction, contains the ground truth 3D faces and the corresponding 68 landmarks of the first 2,000 AFLW \cite{6130513} samples. BIWI \cite{DBLP:journals/ijcv/FanelliDGFG13} Kinect Head Pose Database is another widely used HPE dataset.

To solve the limited frontal-view challenge from 300W-LP, WHENet \cite{zhou2020whenet} takes advantage of the CMU Panoptic Dataset \cite{Joo_2017_TPAMI}'s extensive 3D facial landmarks captured from 32 cameras spanning an entire hemisphere. It pioneered techniques to transform the Panoptic Dataset into a comprehensive HPE dataset covering full-range of angles. Later HPE papers like 6D-RepNet360 \cite{hempel2023robust} used WHENet's CMU Panoptic pose data generation method/code to improve its previous range limited 6D-RepNet \cite{Hempel_2022}.  

Followed with Hu et al. \cite{hu2024mathematical}'s improvement of WHENet's accuracy in head pose rotation calculation, we enhance WHENet's HPE generation and deliver the highest-quality head pose label generation \emph{CMU\_HPE\_10K} for the Panoptic dataset. It contains 10466 images and their head rotation labels. Although offering full-range (Table \ref{table:euler_range_limits}) and precise head orientations (Section \ref{sec:aug_performance}), \emph{CMU\_HPE\_10K} encounters limitations of consistent background and human head diversity. Replacing backgrounds is easy, but replacing heads is subtle. Head rotations and correspondent images are hard to create. We address this issue using novel formulas derived from rigorous math, effectively generating synthetic head rotations and the corresponding images through Blender \cite{blender} (Section \ref{conv:pyr_2_rpy} and \ref{sec:blender}) and the 3D-aware Style-GAN-based Panohead \cite{an2023panohead}.

\subsection{Head Pose Estimation} 
HPE can be divided into two categories, with and without facial landmarks. 
At first, head orientation only makes sense when facial features are visible. Dlib \cite{dlib09} is a pioneer with landmarks for HPE. But it will fail if only few key points are detected. With facial landmarks, we can mathematically derive the 3D transformations involved in changing the 2D projected shape of facial landmarks with the help of a ``standard'' face. Later, researchers start to utilize CNN to predict the three Euler angles directly, for example, Hopenet\cite{Ruiz_2018_CVPR_Workshops} and WHENet\cite{zhou2020whenet}. However, the Euler-angle loss easily leads to discontinuity (Figure \ref{fig:gimbal_lock}) because most rotations have multiple correct yet different representations. We will discuss this later. To avoid such discontinuity, 6D-RepNet \cite{Hempel_2022} (including 6D-RepNet360\cite{hempel2023robust}) and TriNet \cite{cao2020vectorbased} predict the $3 \times 2$ and $3 \times 3$ rotation matrices respectively. 6D-RepNet still applies the Euler-angle evaluation metric for comparison with other HPE methods. However, due to the natural scarcity of the labeled data, most of these approaches still utilize datasets created/ synthesized from facial key-points approaches. 

\section{Key Mathematical Definitions}
To extend constrained-range HPE datasets to the full-range head rotation datasets, we must have precise mathematical knowledge on rotation and Euler angles used in later stages of the paper.

\begin{defn} \label{def:rotation_matrix}
A 3D \textbf{rotation matrix} in a 3D coordinate system is a matrix $R \in SO(3)$, where $SO(3)$ is the group of invertible 3 $\times$ 3 matrices such that $det(R) = 1$ and $R \times R^{T} = R^{T} \times R = I_{3}$.  Each column of R represents each rotated axis so that R can be used to rotate objects. Then for every (row) vector $\bm{v} \in \mathbb{R}^{3}$, We assume the application of the rotation is to the left of the vector, i.e. the rotated vector $\bm{v}_{rotated}$ is defined as follows: $\bm{v}_{rotated} = (R \times \bm{v}^{T})^{T}$. 
\end{defn}

\begin{defn} \label{def:extrinsic_def}
Two coordinate systems are widely used: \textbf{intrinsic coordinate system}, where the origin lies at the center of the head and moves/ rotates along with the head; \textbf{extrinsic coordinate system}, where properties coincide with the intrinsic coordinate system initially but does not move along with the head.  The 3D rotation R defined earlier can also be refined by \textbf{extrinsic} and \textbf{intrinsic} rotations accordingly.  \textbf{Intrinsic rotations} \cite{wiki_euler_angle_2023} are elemental rotations that occur about the axes of the intrinsic coordinate system. \textbf{Extrinsic rotations} \cite{wiki_euler_angle_2023} are elemental rotations that occur about the axes of the extrinsic coordinate system. The \textbf{intrinsic coordinate system} remains fixed from the perspective of a targeted human head. In contrast, the \textbf{extrinsic coordinate system} remains fixed from the perspective of external observers.  

The \textbf{Euler angles} \cite{bhlitem38593} are the three elemental angles \textbf{yaw, pitch, and roll} representing the orientation of a rigid body to a fixed coordinate system. We use \textbf{handedness} (the right-handed or left-handed rule) along each axis to define the positive direction of each Euler angle. Note that this is different from the left-handed or right-handed coordinate systems for Gaming engines. Users also need to define the desired order to recover the rotation matrices uniquely. 

To avoid ambiguity, we define $\bm{yaw=0, pitch=0}$, and $\bm{roll=0}$ to represent the straight frontal face as shown on the left in Figure \ref{fig:300w-lp-system2} and restrict $\bm{(-\pi, \pi]}$ to be \textbf{the range of yaw, roll, and pitch}. 

E. Bernardes's \textit{Quaternion to Euler angles conversion} \cite{Bernardes_2022_pone} decomposes the 3D rotation matrix into the product of three elemental (axis-wise) rotations:
\begin{equation}
    \label{formula:40}
    R = R_{\bm{e_{3}}}(\theta_{3}) \times R_{\bm{e_{2}}}(\theta_{2}) \times R_{\bm{e_{1}}}(\theta_{1})
\end{equation}
where $R_{\bm{e_{i}}}(\theta)$ represents a rotation by the angle $\theta$ around the unit axis vector $\bm{e_{i}}$. These consecutive axes must be orthogonal, i.e., $\bm{e_{1}} \cdot \bm{e_{2}} = \bm{e_{2}} \cdot \bm{e_{3}} = 0$.  In addition, $\bm{e_{1}}$, $\bm{e_{2}}$ and $\bm{e_{3}}$ are orthogonal unit vectors and $\bm{e_{3}} = \epsilon \cdot (\bm{e_{1}} \times \bm{e_{2}})$ in which $\epsilon = (\bm{e_{1}} \times \bm{e_{2}}) \cdot \bm{e_{3}} = \pm1$.
\end{defn}

The following Theorem, cited from \cite{hu2024mathematical}, demonstrates the relation between the two rotation types.

\begin{theorem} \label{thm:diff_rots2}
Suppose $\bm{e_{1}}$, $\bm{e_{2}}$, and $\bm{e_{3}}$ are the three orthogonal unit axis vectors defined in Definition \ref{def:extrinsic_def}. Given the elemental rotation sequence, first rotating $\theta_{1}$ along $\bm{e_{1}}$, then rotating $\theta_{2}$ along $\bm{e_{2}}$, last rotating $\theta_{3}$ along $\bm{e_{3}}$,  the intrinsic and extrinsic rotations, denoted by $R_{intrinisc}$ and $R_{extrinisc}$ respectively, can be expressed as follows: 
\begin{equation}
\label{formula:53}
\begin{split}
R_{intrinsic} &= R_{\bm{e_{1}}}(\theta_{1}) \times R_{\bm{e_{2}}}(\theta_{2}) \times R_{\bm{e_{3}}}(\theta_{3}) \\
R_{extrinsic} &= R_{\bm{e_{3}}}(\theta_{3}) \times R_{\bm{e_{2}}}(\theta_{2}) \times R_{\bm{e_{1}}}(\theta_{1}) 
\end{split}
\end{equation}
\end{theorem}

Based on Theorem \ref{thm:diff_rots2} and the fact that matrix multiplications are non-commutative, we can conclude (1) the intrinsic and extrinsic rotation matrices are in reverse order; and (2) \textbf{\emph{the axis-sequence}} (the rotating sequence of the three axes) matters for identifying the multiplicative order of the elemental rotation matrices to compute the real rotation matrix $R$. For example, the intrinsic XYZ-sequence rotation yields   $R_{X} \times R_{Y} \times R_{Z}$ while the intrinsic ZYX-sequence rotation yields $R_{Z} \times R_{Y} \times R_{X}$.
 
So we suggest HPE dataset creators give \textbf{\emph{well-posed rotation system for HPE}} which should composed of the following: (1) Specify the 3D Cartesian coordinate system used, (2) Each axis's handedness and associated Euler angle (if applicable), (3) Range of yaw, pitch, and roll (if applicable), and (4) The yaw, pitch, and roll rotation sequence (if applicable). For simplicity, we also refer a coordinate (or rotation) system A adapting B's system to A's system following B's system disregarding transitions and scaling.   

\section{The 300W-LP Dataset}
300W-LP's Euler-angle labels are widely used for papers on HPE and 3D dense face alignment \cite{DBLP:conf/cvpr/ZhuLLSL16, 3ddfa_cleardusk, guo2020towards, Hempel_2022, hempel2023robust}. 
However, throughout the paper we are unable to find clear definition of coordinate system used. This motivated us to find out which coordinate and rotation system align to 300W-LP's rotation and pose definitions.  The subsequent sections detail the step-by-step process of unraveling the mystery.

\subsection{Derivation of 300W-LP's coordinate system, rotation, and Euler angle definitions} \label{sec:derive_300wlp}
While 300W-LP's code \cite{300w_lp_url} adopts the Base Face Model (BFM) with neck \cite{bfm}, its successor \cite{DBLP:conf/cvpr/ZhuLLSL16} uses the BFM without the ear/neck \footnote{\url{https://github.com/cleardusk/3DDFA_V2/issues/101}} region. Because both BFMs use the same coordinate system, We opted for the latter option as the BFM (since the BFM link provided by \cite{300w_lp_url} is no longer valid), to establish the rotation system for 300W-LP.
From the no-neck BFM, we can identify the coordinate system used in 300W-LP. Next, 300W-LP's code provides the elemental rotation matrices for the axes-sequence and Euler angles. See Formula (\ref{eq:16}) for details. Note that $R^{left}_{X}(p)$ means rotation about $X$-axis in left-handed-rule fashion. 
\begin{equation*}
    R^{left}_{X}(p) = \begin{bsmallmatrix}
    1 & 0 & 0\\
    0 & cos(p) & sin(p)\\
    0 & -sin(p) & cos(p)
    \end{bsmallmatrix}, 
    R^{left}_{Y}(y) = \begin{bsmallmatrix}
    cos(y) & 0 & -sin(y)\\
    0 & 1 & 0\\
    sin(y) & 0 & cos(y)
    \end{bsmallmatrix}, 
    R^{left}_{Z}(r) = \begin{bsmallmatrix}   
    cos(r) & sin(r) & 0\\
    -sin(r) & cos(r) & 0\\
    0 & 0 & 1
    \end{bsmallmatrix} 
\end{equation*}
\begin{equation}
  \label{eq:16}
  \begin{split}
    R_{W}(y, p, r) &= R^{left}_{X}(p) \times R^{left}_{Y}(y) \times R^{left}_{Z}(r) \\
    & = \begin{bsmallmatrix}
    cos(y)cos(r) & cos(y)sin(r) & -sin(y) \\
    -cos(p)sin(r)+sin(p)sin(y)cos(r) & cos(p)cos(r)+sin(p)sin(y)sin(r) & sin(p)cos(y) \\
    sin(p)sin(r)+cos(p)sin(y)cos(r) & -sin(p)cos(r)+cos(p)sin(y)sin(r) & cos(p)cos(y) 
    \end{bsmallmatrix}
  \end{split}
\end{equation}
Combining the coordinate system,  elemental rotations, and intrinsic XYZ-sequence (i.e., pitch-yaw-roll order), we conclude that 300W-LP's Euler angles adapt the left-handed rule rotation. Figure \ref{fig:300w-lp-system2}, shown below, demonstrates our inferred definitions for the coordinate system and Euler angle used in 300W-LP's annotations.
\begin{figure}[H]
\centering
\includegraphics[width=1.0\textwidth]{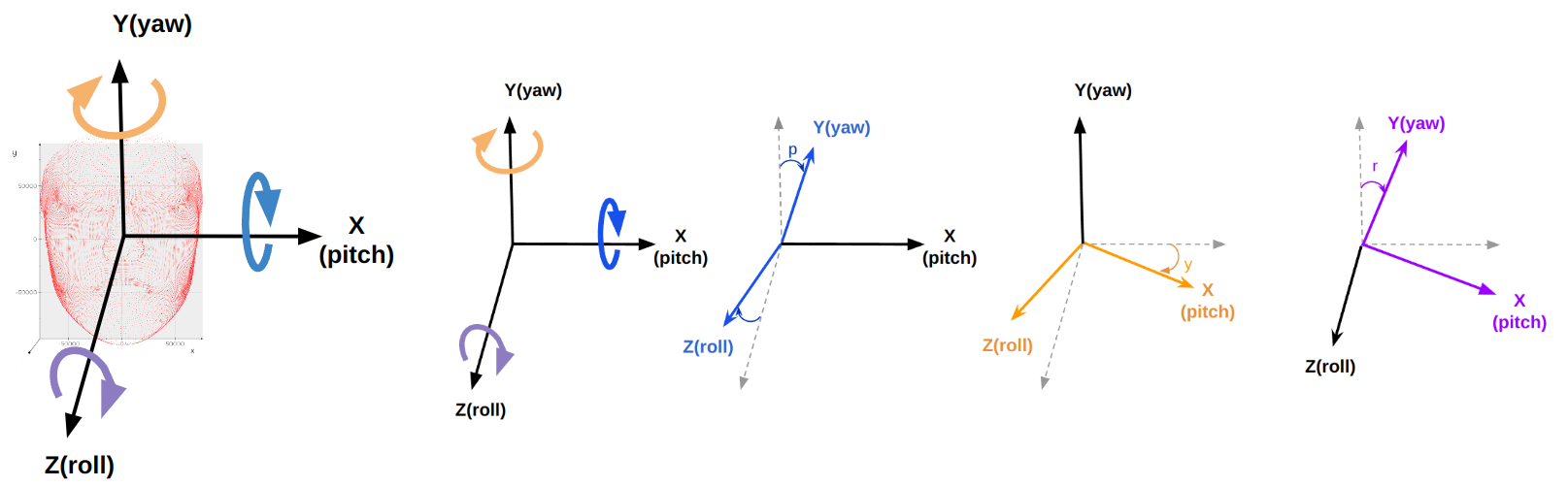}
\caption{The leftmost image depicts 300W-LP’s coordinate system and Euler angles adapts the left-handed rule. Other images illustrate 300W-LP's elemental rotations}
\label{fig:300w-lp-system2}
\end{figure}
\begin{defn}\label{def:300wlp_rotation}
\textbf{300W-LP's rotation system} 
So we infer that 300W-LP adapts the left-handed coordinate system and Euler angle definition shown in Figure \ref{fig:300w-lp-system2} and the intrinsic XYZ axis-sequence (or pitch-yaw-roll order) with Formula (\ref{eq:16}).    
\end{defn}

\subsection{Axes-line Drawing Fueled by 300W-LP-compatible Head Rotations}
\label{our_300wlp_drawing_section}
Later works such as 3DDFA{\_}v2 adopt the scale orthographic projection to draw the 3D pose frustums \cite{3ddfa_cleardusk}, but the most popular drawing method is HopeNet's \cite{Ruiz_2018_CVPR_Workshops} implementation, \textit{draw{\_}axis}() \footnote{\url{https://github.com/natanielruiz/deep-head-pose/blob/master/code/utils.py}} (shown in Figure \ref{fig:draw_axis}). It has been widely used for many HPE works such as WHENet \cite{zhou2020whenet}, 6D-RepNet \cite{Hempel_2022} etc. 

\begin{defn}
For the axes-line drawing, we follow the convention: the \textbf{red line} stretching toward the left side of the frontal face, the \textbf{green line} stretching from the forehead to the center of the jaw, and the \textbf{blue line} pointing from the nose out.  The three lines are perpendicular to each other in the 3-dimensional space, even though sometimes they don't appear so due to the 2D image projection.  
\end{defn}

To directly derive three-axis drawing of a rotation matrix $R_{W}(y, p, r)$ defined in Definition \ref{def:300wlp_rotation}, we first apply a coordinate transformation on $R_{W}(y, p, r)$ to the image coordinate system with the Y-axis pointing downwards. Then project the three column vectors of the resulted rotation $R^{\textbf{draw}}_{W}$ onto the image plane (i.e., XY-plane) to obtain the three RGB lines. 
\begin{equation}
  \label{formula:30}
  \begin{split}
  &R^{\textbf{draw}}_{W} = T_{W} \times R_{W}(y, p, r) \times T^{-1}_{W}, \quad where \quad 
   T_{W} = \begin{bsmallmatrix}
    1 & 0 & 0  \\
    0 & -1 & 0 \\
    0 & 0 & 1
    \end{bsmallmatrix} \\
  &R_{proj} = \begin{bsmallmatrix}   
               1 & 0 & 0\\
               0 & 1 & 0
              \end{bsmallmatrix}  
              \times T_{W} \times R_{W}(y, p, r) \times T^{-1}_{W}
    \end{split}
\end{equation}
\noindent
In Appendix \ref{appendix:300wlp_3line_drawing}, we thoroughly compared our proposed Formula (\ref{formula:30}) (of rotation matrices as inputs) with \textit{draw{\_}axis}() (of pitch-yaw-roll as inputs) and observe that they output the same axes-line drawings.  We also find that Formula (\ref{formula:30}) relaxes \textit{draw{\_}axis}()'s pitch-yaw-roll input limitation because \textbf{\emph{the Euler angles are not necessary for the drawing if the rotation matrix is provided}}. 

\subsection{Extracting the Euler Angles from A Rotation Matrix in 300W-LP's Rotation System}
\label{subsec:extract_300w_euler}
After identifying the rotations utilized in 300W-LP follow Formula (\ref{eq:16}), we can compute the Euler angles from it. We show the complete derivation of our Euler angle extraction in Appendix \ref{appendix:300wlp_euler_extraction}, and present the algorithm (Figure \ref{fig:300wlp_extraction}) for computing the closed-form yaw, roll, and pitch for rotations defined in Definition \ref{def:300wlp_rotation}. Please note that there are more than one correct \emph{pitch-yaw-roll} solutions in most rotations. For details, please refer to Appendix \ref{appendix:300wlp_euler_extraction}.

Issues arise in cases of Gimbal lock. Take the following image from 300W-LP for example, its pitch-yaw-roll label is (-16.090911401458296\degree, -89.9985818251308\degree, -6.854511900533989\degree) and shown leftmost in Figure \ref{fig:gimbal_lock}. Even though the labeled yaw $\ne -90\degree$, its corresponding rotation matrix $M$ does meet the Gimbal lock requirement, i.e., $y = -90\degree$ because the Euler angles extracted from $M$ yield the general pitch-yaw-roll solutions, $(p, -90\degree, r)$ and $p+r=-22.94542388660367\degree$. The right three images in Figure \ref{fig:gimbal_lock} show three possible solutions of Euler angles for the same rotation. So \textbf{\emph{directly learning the Euler angles in this case is a bad idea, and we suggest learning from the rotation matrix representations for the full-range HPE}}.  
\begin{figure}[H]
\centering
\includegraphics[width=0.9\textwidth]{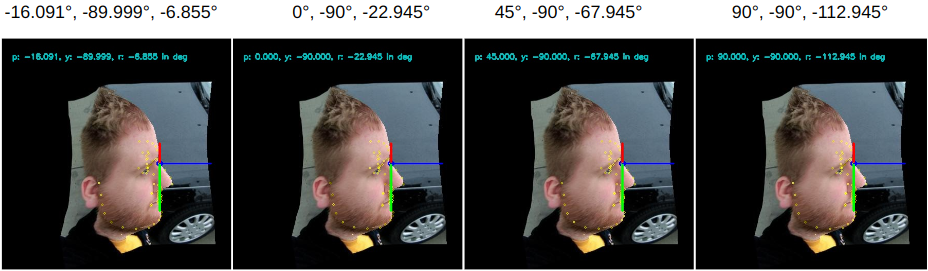}
\caption{A 300W-LP's Gimbal Lock example.}
\label{fig:gimbal_lock}
\end{figure}
Finally, our Euler angle extraction, Figure \ref{fig:300wlp_extraction} in Appendix \ref{appendix:300wlp_euler_extraction}, presents two sets of pitch-yaw-roll solutions for non-Gimbal-lock case. As to which one 300W-LP selected as labels,  according to the dataset's labeled yaw $\in [-90\degree, 90\degree]$, and $y_{1} = \bm{arcsin}(-R_{0,2}) \in [-90\degree, 90\degree]$,  we infer that 300W-LP picked the first.

\section{2D Image Data Augmentation for Head Pose Dataset}\label{sec:data_aug}
Data augmentation is crucial in the training pipeline for computer vision models, enabling them to achieve higher performance without increasing the number of training images. However, in HPE, 2D geometric image augmentations, rotations and flipping, have been avoided due to (1) a lack of clear definition of rotation systems and (2) mathematical derivation of new rotation matrices under 2D geometric augmentations. Obviously, \textbf{\emph{geometric augmentation also enhances angle coverage for free}}. In contrast, pixel-wise image augmentations like Gaussian blur do not modify labeled rotations by any means. In Appendix \ref{appendix: 3dflip_rot},  we derive Theorem \ref{thm:300wlp_2d_rotation} and Corollary \ref{cor:300wlp_flipcor} to get the augmented rotations for 2D geometric transformations. To showcase the effectiveness of our derivations, we applied Corollary \ref{cor:300wlp_flipcor} to different head pose images, as depicted in Figure \ref{fig:pose_augmentation}.

\begin{theorem} \label{thm:300wlp_2d_rotation}
Let's fix to 300W-LP's rotation system. Suppose an image contains a human head, and the head's intrinsic rotation $R$ is given. 
\textbf{Rotating the image by $\phi$ counter-clockwise} results in the roll's elemental rotation $R^{left}_{Z}(-\phi)$ defined in Formula (\ref{eq:16}). Then we have that $R_{rot{\_}img}(\phi)$ satisfies the following formula. 
\begin{equation}
  \label{formula:1}
  \begin{split}
    &R_{rot{\_}img}(\phi)
    =  \begin{bsmallmatrix}
        cos(\phi) & -sin(\phi) & 0\\
        sin(\phi) & cos(\phi) & 0\\
        0 & 0 & 1
       \end{bsmallmatrix} \times R
  \end{split}
\end{equation}
The \textbf{3D rotation $R_{flip}(\theta)$, corresponding to the 2D flipping across the line $L_{\theta}$}, can be expressed as follows:
\begin{equation}
  \label{formula:2}
    R_{flip}(\theta) = \begin{bsmallmatrix}
        cos(2\theta) & sin(2\theta) & 0\\
        sin(2\theta) & -cos(2\theta) & 0\\
        0 & 0 & 1
       \end{bsmallmatrix}  \times R \times \begin{bsmallmatrix}
        -1 & 0 & 0\\
        0 & 1 & 0\\
        0 & 0 & 1
       \end{bsmallmatrix}     
\end{equation}
\end{theorem}

\begin{corollary}\label{cor:300wlp_flipcor} Special cases of Theorem \ref{thm:300wlp_2d_rotation}:
    \begin{enumerate}
    \item Horizontal flipping: $R_{flip}(\pi/2) =  \begin{bsmallmatrix}
                        -1 & 0 & 0\\
                        0 & 1 & 0\\
                        0 & 0 & 1
                       \end{bsmallmatrix}  \times R \times \begin{bsmallmatrix}
                        -1 & 0 & 0\\
                        0 & 1 & 0\\
                        0 & 0 & 1
                       \end{bsmallmatrix}$,
    \item Vertical flipping: $R_{flip}(0) =  \begin{bsmallmatrix}
            1 & 0 & 0\\
            0 & -1 & 0\\
            0 & 0 & 1
           \end{bsmallmatrix}  \times R \times \begin{bsmallmatrix}
            -1 & 0 & 0\\
            0 & 1 & 0\\
            0 & 0 & 1
           \end{bsmallmatrix}$. 
    \item Flipping about both axes:
    $R_{bf} = \begin{bsmallmatrix}
            -1 & 0 & 0\\
            0 & -1 & 0\\
            0 & 0 & 1
           \end{bsmallmatrix}  \times R,$  
    \item The symmetrical flipping about $\bm{L_{\pi /4}}$:  
    $R_{flip}(\pi /4) = \begin{bsmallmatrix}
            0 & 1 & 0\\
            1 & 0 & 0\\
            0 & 0 & 1
           \end{bsmallmatrix}  \times R \times \begin{bsmallmatrix}
            -1 & 0 & 0\\
            0 & 1 & 0\\
            0 & 0 & 1
           \end{bsmallmatrix} ,$ \quad and
    \item Rotating image $\pi/4$ counter-clockwise:
            $R_{rot}(\pi/4) = \begin{bsmallmatrix}
            cos(\pi/4) & -sin(\pi/4) & 0\\
            sin(\pi/4) & cos(\pi/4) & 0\\
            0 & 0 & 1
           \end{bsmallmatrix}  \times R$.     
\end{enumerate}  
\end{corollary} 

\subsection{Pitch-yaw-roll to Roll-pitch-yaw Conversion in 300W-LP Coordinate System}
\label{conv:pyr_2_rpy}
During our synthetic HPE data generation, we observed that the resulting Euler angle axis-sequence is roll-pitch-yaw. Appendix \ref{appendx:pyr_2_rpy} presents the solution (Fig. \ref{listing:1}) of converting from 300W-LP's intrinsic pitch-yaw-roll order to intrinsic roll-pitch-yaw order. Fig. \ref{listing:1} and rotation augmentations through Formula (\ref{formula:1}) guarantees the nonzero-roll head pose synthetic image generation is capable of mimicking any real image's pitch-yaw-roll head pose through Blender or Panohead (demonstrated in Figure \ref{fig:pano_cmu1}). This Euler angle conversion will be used later in our experiments and implies that for dataset creators, \textbf{\emph{a good coverage in pitch and yaw are enough and our augmentation techniques fill up the gaps in the missing distribution of rolls}}. 
Please note Appendix \ref{appendx:pyr_2_rpy} also allows the following: \emph{given an existing head pose image with orientation $R$, we can fabricate a synthetic image with orientation exact equals $R$ using Blender 3D engine or Panohead rendering.} This easily facilitates the creation of a positive sample given an anchor and opens the door in the context of contrastive learning.

\section{Experiments}
We present two experiments to show the usefulness of our mathematical derivations. The first is to generate a wide pitch-yaw range of synthetic head-region images with 300W-LP-compatible rotation matrix labels, and we can ``densify'' the roll range of the dataset with technique, Fig. \ref{listing:1} in Appendix \ref{appendx:pyr_2_rpy}. The second is to show our flipping/ rotation augmentations enhance standard head pose estimation networks' performance.

\subsection{A Toy Synthetic HPE Dataset with Blender: Pitch-Yaw Coverage is All You Need}
\label{sec:blender}
Figure \ref{fig:blender_setting} demonstrates how we generate human head images and camera extrinsics through the Blender Python script. We put a human head mesh model in the center and set camera to scan along a spiral path to ensure good pitch-yaw coverage. Then we convert Blender's camera extrinsic outputs into 300W-LP-compatible rotations. But this process only produces zero-roll rotations. So we use the Euler angle conversion mentioned in Section \ref{conv:pyr_2_rpy} to augment and increase the roll coverage.  

To demonstrate our point, we generate 1440 snapshot images illustrated in Figure \ref{fig:blender_setting}. Then we augment the 1440 pitch-yaw-only images with random rotation/ flipping angles to 2880 ``augmented'' images.  We also randomly sample 1440 rotation matrices through \textbf{\emph{scipy.spatial.transform.Rotation.random}} package \cite {2020SciPy-NMeth}. Followed by inputting a total of $1440 + 2880 + 1440$ samples into \textbf{\emph{TensorBoard}} \cite{tensorflow2015-whitepaper}. The result is shown in Figure \ref{fig:tsne}.

We observe that our augmentations significantly extended the limited range of Blender sampled head poses (the orange dots as shown in the middle left) to nearly any possible random rotations (shown in the middle right). Our augmentations also covered the full-range rotations. To our surprise, the augmented rotations exhibit a significantly uniform distribution under Principle Component Analysis (PCA) compared to random rotations, as illustrated in Figure \ref{fig:tsne}.

\begin{figure}[H]
\centering
\includegraphics[width=0.4\textwidth]{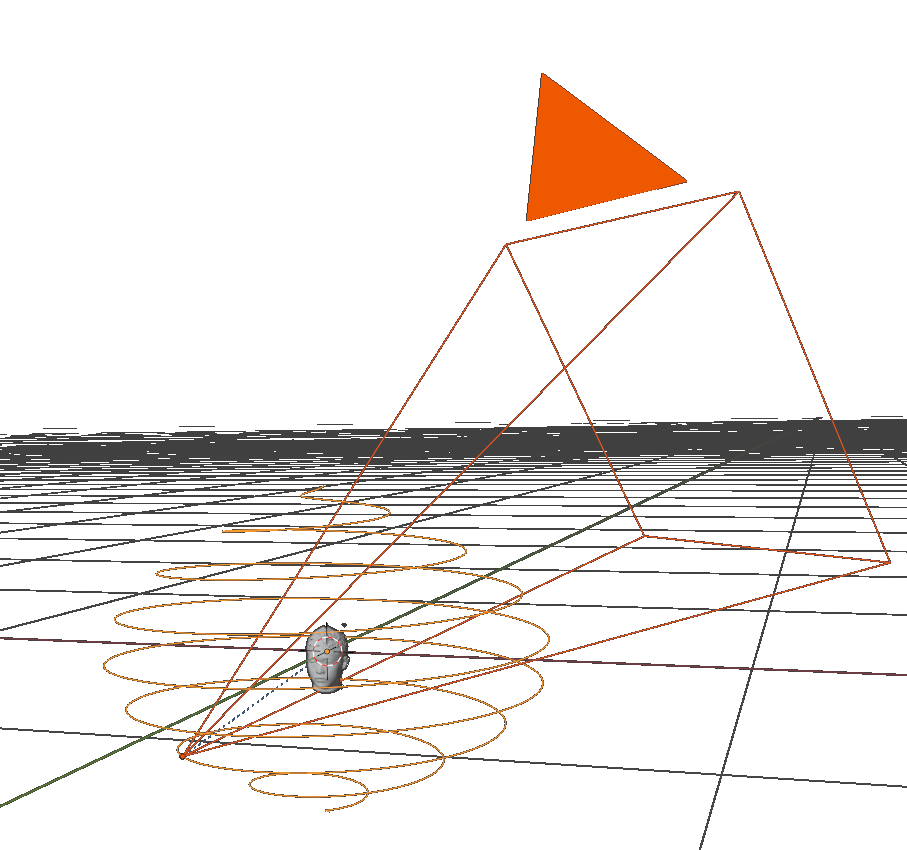}
\caption{The human head mesh model is in center, the camera moves along the spiral path and 1440 snapshots are taken with their rotation matrices recorded. During this process, the camera has no rotation so all images generated has zero rolls for the head pose.}
\label{fig:blender_setting}
\end{figure} 

\begin{figure}[H]
\centering
\includegraphics[width=0.8\textwidth]{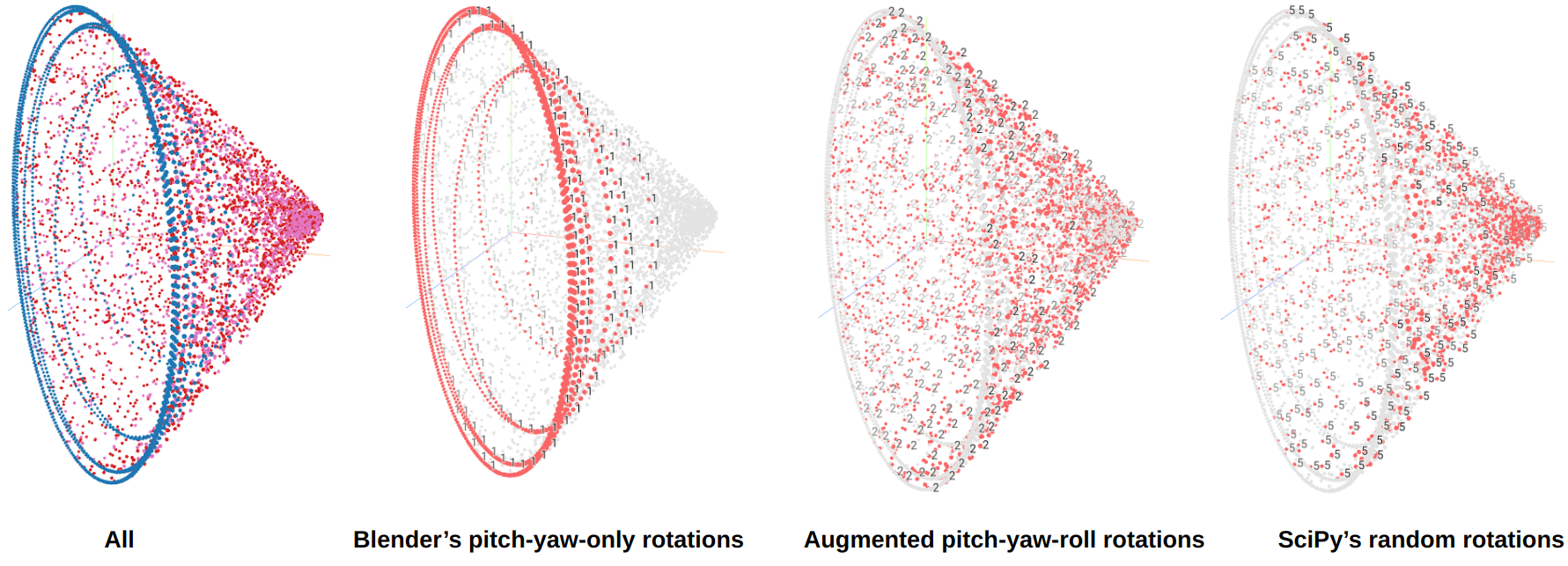}
\caption{PCA visualization for rotations. All $3 \times 3$ matrices are flattened to 9-dim vectors and projected to 3D by PCA. The leftmost shows all rotation samples, and the blue spiral dots represent Blender's generated rotations.}
\label{fig:tsne}
\end{figure}

\subsection{Flipping/Rotation Augmentations Improve Existing Network Performance}
\label{sec:aug_performance}
To the best of our knowledge, 6D-RepNet was the first to incorporate the Geodesic loss in head pose neural network training. However, Mean absolute value (MAE) was still used as evaluation metric (for comparing with Euler-angle labeled ground truths) for the mean pitch-yaw-roll values. Let's re-state the definition first.
\begin{defn}
Define the \textbf{Geodesic distance} between two rotation matrices  $A, B \in SO(3)$ to be the following:
\begin{equation}
    \label{gedesic_dist}
    d(A, B) = cos^{-1}(\frac{tr(A \times B^{T})-1}{2}) 
\end{equation}
where $tr$ and $T$ represent the trace and transpose, respectively.  
\end{defn}
This definition precisely captures the angle (in radian) between every pair of 3D rotation matrices since geodesics on the unit sphere are sections of big circles directly corresponding to angles to the center of the sphere. So we adopt \textbf{\emph{mean geodesic error}} as the evaluation metric since it avoids the discontinuity problems caused by the Gimbal locks that the conventional MAE evaluation metric suffers. 

However, we notice that 6D-RepNet's \cite{Hempel_2022} implementation restricts the Euler angles within $\pm99\degree$. Furthermore, Hu et al. \cite{hu2024mathematical} points out that 6D-RepNet's derived rotation matrices $R$, are actually $R^t$ (although at the Euler angle extraction stage, $R$ was implicitly transposed). Thus we modified 6D-RepNet's matrix computation according to Formula (\ref{eq:16}), replaced its Euler-angle extraction with Appendix \ref{appendix:300wlp_euler_extraction}, and added our 2D geometric augmentation feature described earlier in Section \ref{sec:data_aug}, and slightly revised the dataset implementation to enable more dataset comparison. We call this new version \textbf{\emph{6D-RepNet+}} as it can fully utilize the Geodesic distance (Definition \ref{gedesic_dist}) and mean geodesic error not only for the training but also for evaluation of pitch-yaw-roll ground-truth labels. It also allows the training with any \textbf{\emph{augmented}} general head pose orientation (e.g. upside down). 

Please note we do not change the RepVGG \cite{ding2021repvgg} backbone, or any other part of 6D-RepNet's neural network, so 6D-RepNet+ is not a competing network architecture. Our intention is to easily facilitate our experiments, taking into account the 6D-RepNet360\cite{hempel2023robust}'s training code was not released, and we also aim to standardize all computations under 300W-LP's rotation system.

We conducted all 6D-RepNet+'s experiments on NVIDIA V100 GPU with 32 GB of RAM, and all 6D-RepNet+ models are trained on the 300W-LP dataset with the batch size = 250, number of epochs = 40, learning rate = 0.0001, $1\%$ Gaussian Blur, and the mean geodesic loss. To fairly compare with 6D-RepNet without head pose augmentation, we set the angle parameter $\bm{0\degree}$ representing 6D-RepNet's training without the (vertical) flipping. For every non-zero angle parameter, e.g. $\bm{20\degree}$, 6D-RepNet+ performs $50\%$ rotation with Formula (\ref{formula:1}) of a randomly chosen angle within $\pm20\degree$, and $50\%$ vertical flipping with Formula (\ref{formula:2}) of the other randomly chosen angle within $[90\degree-20\degree, 90\degree]$ to slightly alter the left-right flipping augmentation a little. 

\begin{table}[H]
    \caption{Range of Euler angles for datasets we used in our experiments.}
    \centering
    \resizebox{0.6\textwidth}{!}{%

    \begin{tabular}{ccccc}
         \toprule
         \textbf{Dataset}  & \textbf{pitch}  & \textbf{yaw} & \textbf{roll} & \textbf{size}\\
         \midrule
         300W-LP limited-range  &  (-99\degree, 97\degree) & [-90\degree, 90\degree] & (-98\degree, 98\degree) & 122,415 \\
         \midrule
         300W-LP full-range     &  (-167\degree, 107\degree) & [-90\degree, 90\degree] & (-147\degree, 129\degree) & 122,450\\
         \midrule
         AFLW2000 limited-range & (-95\degree, 89\degree) & (-84\degree, 87\degree)   & (-91\degree, 93\degree)  & 1,969  \\
         \midrule
         AFLW2000 full-range & (-152\degree, 163\degree) & (-86\degree, 87\degree)   & (-162\degree, 154\degree)   & 2,000 \\
         \midrule
         BIWI              & (-63\degree, 78\degree) & (-75\degree, 77\degree) & (-57\degree, 45\degree) & 13,219\\
         \midrule
         BIWI\_Aug          & (-77\degree, 74\degree) & (-72\degree, 76\degree\degree) & (-54\degree, 104\degree)  & 13,219\\
         \midrule
         CMU\_HPE\_10K       & (-180\degree, 180\degree]      & [-90\degree, 90\degree]   & (-180\degree, 180\degree]    & 10,466 \\
         \bottomrule
    \end{tabular}
    }
    \label{table:euler_range_limits}
\end{table}
Most common HPE datasets are limited-range, as shown in Table \ref{table:euler_range_limits}. For our second experiment, we prepare two versions of 300W-LP and AFLW2000 (note we didn't do this for BIWI as its Euler-angle range is already very limited), \textbf{\emph{limited}}-range and \textbf{\emph{full}}-range datasets. The limited-range datasets only provide the limited head pose range, all Euler angles are within $\pm99\degree$, and the full-range datasets remove this restriction.  To compensate the head pose range imposed by 300W-LP, AFLW2000, and BIWI, we created two datasets, BIWI{\_}Aug and CMU{\_}HPE{\_}10K. We construct \textbf{\emph{BIWI{\_}Aug}} dataset by applying up to $20\degree$ augmentation described above on the BIWI dataset. And the CMU{\_}HPE{\_}10K dataset covers the complete and uniform head rotation spectrum.  Our experiment results with 6D-RepNet+ are shown in Table \ref{6d+_epxs}.

In the process of CMU{\_}HPE{\_}10K creation, we adapted Formula (\ref{formula:hu_formula}) \cite{hu2024mathematical} and applied a rotation $E_{ref}$ (instead of the coordinate transformation used by WHENet's rotation generation) on $C_{extr} \times R_{Horn}$ to map the computed rotation to the image space. We found this formula greatly improved the head pose extraction. More details in Figure \ref{fig:cmu}.
\begin{equation}
    \label{formula:hu_formula}
    \begin{split}
    &E_{ref} =\begin{bsmallmatrix}
            1 & 0 & 0 \\
            0 & -1 & 0 \\
            0 & 0 & -1
          \end{bsmallmatrix}, \\
    &R_{panoptic} \coloneqq E_{ref} \times C_{extr} \times R_{Horn}.
    \end{split}
\end{equation}
where $C_{extr}$ is the camera extrinsic and $R_{Horn}$ is the rotation extracted by the Horn's method \cite{article3} from the base face model to OpenPose's \cite{8765346} face model. 
\begin{figure}[H]
\centering
\includegraphics[width=1.0\textwidth]{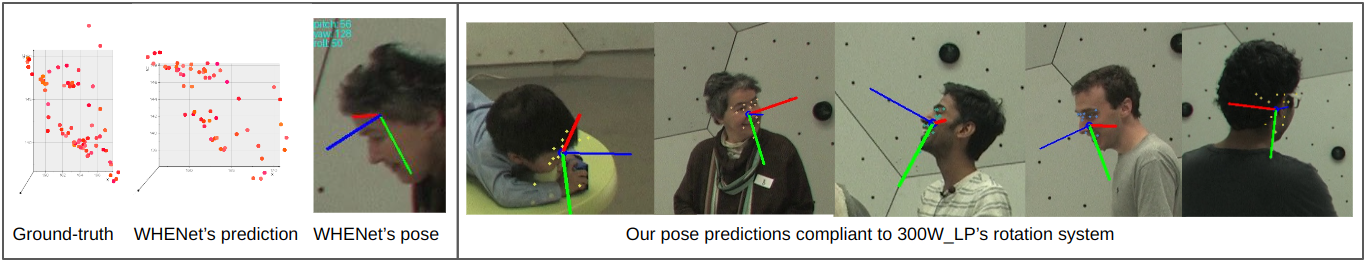}
\caption{From \cite{hu2024mathematical}, the red line on the left panel with WHENet's method should stretch to the right, and the right panel with Formula \ref{formula:hu_formula} greatly improves the head pose extraction.}
\label{fig:cmu}
\end{figure}

\begin{table}[H]
\caption{Our augmented rotation and flipping experiments. We use the mean geodesic error as the evaluation metric,  so the evaluated errors are in radians. \textbf{Boldface} denotes the best-performing augmentation under limited or full-range setting for the specific dataset. For the improvements, radian has been converted to degree. }
\resizebox{\textwidth}{!}{
\centering
\begin{tabular}{c|c|cc|cc|cc|cc|cc}   
    \toprule
    \multirow{2}{*}{\textbf{Angle}} & \multirow{2}{*}{\textbf{Setting}} & \multicolumn{2}{c}{\textbf{AFLW2000-limited}}   & \multicolumn{2}{c}{\textbf{AFLW2000-full-range}}    & \multicolumn{2}{c}{\textbf{BIWI}}&   \multicolumn{2}{c}{\textbf{BIWI\_aug}}   & \multicolumn{2}{c}{\textbf{CMU\_HPE\_10K}} \\
    \cmidrule(lr){3-4} \cmidrule(lr){5-6} \cmidrule(lr){7-8} \cmidrule(lr){9-10} \cmidrule(lr){11-12}
     & & \textbf{Radian} & \textbf{Improv.} & \textbf{Radian} & \textbf{Improv.} & \textbf{Radian} & \textbf{Improv.} & \textbf{Radian} & \textbf{Improv.} & \textbf{Radian} & \textbf{Improv.} \\
     \midrule
     \textbf{0\degree} & limited & 0.1147 & - & 0.1215 & - & 0.1222 & - & 0.7251 & - & 0.8216 & - \\
     \textbf{0\degree} & full & 0.1153 & - & 0.1228 & - & 0.1263 & - & 0.7279 & - & 0.7939 & - \\
     \midrule
     \textbf{3\degree} & limited & 0.1054 & 0.5339\degree & 0.1126 & 0.5109\degree & \textbf{0.1178} & \textbf{0.2525}\degree & 0.7099 & 0.8684\degree & 0.7703 & 2.9431\degree \\
     \textbf{3\degree} & full & 0.1085 & 0.3896\degree & 0.1155 & 0.4146\degree & 0.1240 & 0.1300\degree & 0.6713 & 3.2432\degree & 0.7582 & 2.0479\degree \\
     \midrule
     \textbf{6\degree} & limited & 0.1064 & 0.4781\degree & 0.1118 & 0.5545\degree & 0.1304 & -0.4681\degree & 0.7138 & 0.6442\degree & 0.7253 & 5.5213\degree \\
     \textbf{6\degree} & full & 0.1093 & 0.3452\degree & 0.1159 & 0.3918\degree & 0.1277 & -0.0818\degree & 0.6859 & 2.4107\degree & 0.7186 & 4.3162\degree \\
     \midrule
     \textbf{15\degree} & limited & \textbf{0.1054} & \textbf{0.5372}\degree & 0.1113 & 0.5879\degree & 0.1261 & -0.2208\degree & 0.6637 & 3.5163\degree & 0.6548 & 9.5659\degree \\
     \textbf{15\degree} & full & 0.1083 & 0.4005\degree & 0.1132 & 0.5488\degree & 0.1299 & -0.2064\degree & 0.6738 & 3.0998\degree & 0.6403 & 8.8036\degree \\
     \midrule
     \textbf{20\degree} & limited & 0.1087 & 0.3446\degree & 0.1145 & 0.4012\degree & 0.1360 & -0.7889\degree & 0.6747 & 2.8895\degree & \textbf{0.5488} & \textbf{15.6408}\degree \\
     \textbf{20\degree} & full & 0.1060 & 0.5336\degree & \textbf{0.1117} & \textbf{0.6328}\degree & 0.1260 & 0.0161\degree & \textbf{0.6367} & \textbf{5.2259}\degree & 0.6601 & 7.6713\degree \\
    \bottomrule
\end{tabular}%
\label{6d+_epxs}
}
\end{table}

From Table \ref{6d+_epxs}, we clearly see that our general flipping/ rotations augmentations indeed improve training performance except for BIWI above 6\degree augmentations. We argue this is because BIWI has very limited Euler-angle range as in Table \ref{table:euler_range_limits} and the performance of BIWI\_aug supports our point.

\begin{figure}[H]
\centering
\includegraphics[width=1.0 \textwidth]{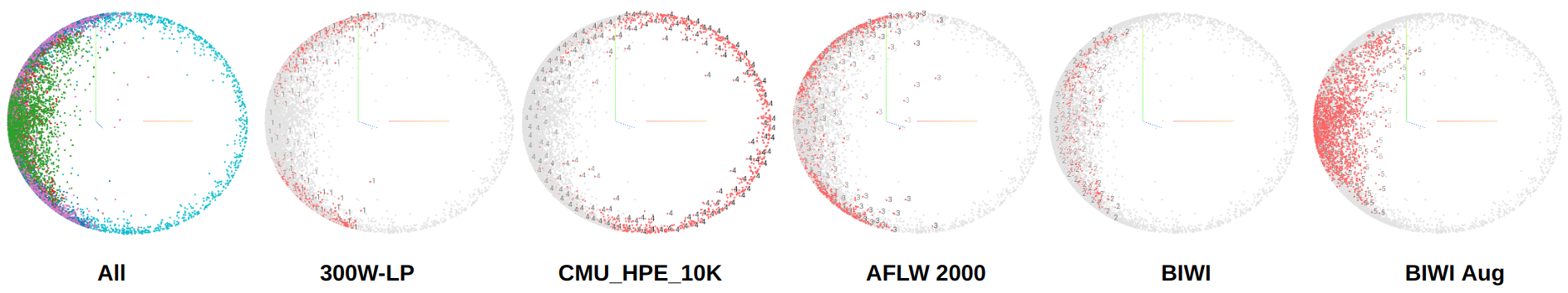}
\caption{PCA of each of the datasets we used in our experiments. Apparently, BIWI has an even smaller distribution profile than 300W-LP, which is enlarged by BIWI\_aug.}
\label{fig:4_datasets}
\end{figure} 

\section{Conclusion}
In this paper, we presented method for inferring coordinate system from 300W-LP source code, method for Euler angle application order determination and formulae for extracting precise rotation matrices and the Euler angles. Based on the clarified 300W-LP coordinate system and Euler angle orientations/ application orders,  we derived formula for 2D rotation/ flipping augmentations of the 300W-LP compatible rotation matrices. Additionally, we also provided derivations for the correct drawing routines for rotation matrices that are independent of Euler angles, which addressed the problems that drawing routines are often borrowed from code designed for a different coordinate system.

We demonstrated the usefulness of our result by constructing a toy synthetic head pose dataset and demonstrated both mathematically and experimentally that a dataset with a good pitch-yaw coverage with zero roll can be augmented to (non-zero roll) full-range dataset. Adding augmentations to existing head pose training pipelines also improves performance.

\bibliographystyle{plain}
\bibliography{neurips.bib}

\begin{thebibliography}{10}

\bibitem{wiki_euler_angle_2023}
Wikipedia's euler angles.
\newblock \url{https://en.wikipedia.org/wiki/Euler_angles#Conventions_by_intrinsic_rotations}, Nov 2023.

\bibitem{tensorflow2015-whitepaper}
Mart\'{\i}n Abadi, Ashish Agarwal, Paul Barham, Eugene Brevdo, Zhifeng Chen, Craig Citro, Greg~S. Corrado, Andy Davis, Jeffrey Dean, Matthieu Devin, Sanjay Ghemawat, Ian Goodfellow, Andrew Harp, Geoffrey Irving, Michael Isard, Yangqing Jia, Rafal Jozefowicz, Lukasz Kaiser, Manjunath Kudlur, Josh Levenberg, Dan Man\'{e}, Rajat Monga, Sherry Moore, Derek Murray, Chris Olah, Mike Schuster, Jonathon Shlens, Benoit Steiner, Ilya Sutskever, Kunal Talwar, Paul Tucker, Vincent Vanhoucke, Vijay Vasudevan, Fernanda Vi\'{e}gas, Oriol Vinyals, Pete Warden, Martin Wattenberg, Martin Wicke, Yuan Yu, and Xiaoqiang Zheng.
\newblock {TensorFlow}: Large-scale machine learning on heterogeneous systems, 2015.
\newblock Software available from tensorflow.org.

\bibitem{bhlitem38593}
Imperatorskaia akademiia~nauk (Russia) and Imperatorskaia akademiia nauk i~khudozhestv (Russia).
\newblock {\em Novi commentarii Academiae Scientiarum Imperialis Petropolitanae}, volume t.20 (1775).
\newblock Petropolis, Typis Academiae Scientarum, 1750-76, 1775.
\newblock https://www.biodiversitylibrary.org/bibliography/9527.

\bibitem{an2023panohead}
Sizhe An, Hongyi Xu, Yichun Shi, Guoxian Song, Umit Ogras, and Linjie Luo.
\newblock Panohead: Geometry-aware 3d full-head synthesis in 360$^{\circ}$, 2023.

\bibitem{DBLP:journals/sncs/AspertiF23}
Andrea Asperti and Daniele Filippini.
\newblock Deep learning for head pose estimation: {A} survey.
\newblock {\em {SN} Comput. Sci.}, 4(4):349, 2023.

\bibitem{DBLP:conf/cvpr/BelhumeurJKK11}
Peter~N. Belhumeur, David~W. Jacobs, David~J. Kriegman, and Neeraj Kumar.
\newblock Localizing parts of faces using a consensus of exemplars.
\newblock In {\em The 24th {IEEE} Conference on Computer Vision and Pattern Recognition, {CVPR} 2011, Colorado Springs, CO, USA, 20-25 June 2011}, pages 545--552. {IEEE} Computer Society, 2011.

\bibitem{Bernardes_2022_pone}
Evandro Bernardes and Stéphane. Viollet.
\newblock Quaternion to euler angles conversion: A direct, general and computationally efficient method.
\newblock In {\em PLoS ONE 17(11): e0276302}, November 2022.

\bibitem{1227983}
V.~Blanz and T.~Vetter.
\newblock Face recognition based on fitting a 3d morphable model.
\newblock {\em IEEE Transactions on Pattern Analysis and Machine Intelligence}, 25(9):1063--1074, 2003.

\bibitem{8765346}
Z.~{Cao}, G.~{Hidalgo Martinez}, T.~{Simon}, S.~{Wei}, and Y.~A. {Sheikh}.
\newblock Openpose: Realtime multi-person 2d pose estimation using part affinity fields.
\newblock {\em IEEE Transactions on Pattern Analysis and Machine Intelligence}, 2019.

\bibitem{cao2020vectorbased}
Zhiwen Cao, Zongcheng Chu, Dongfang Liu, and Yingjie Chen.
\newblock A vector-based representation to enhance head pose estimation, 2020.

\bibitem{blender}
Blender~Online Community.
\newblock {\em Blender - a 3D modelling and rendering package}.
\newblock Blender Foundation, Stichting Blender Foundation, Amsterdam, 2018.

\bibitem{ding2021repvgg}
Xiaohan Ding, Xiangyu Zhang, Ningning Ma, Jungong Han, Guiguang Ding, and Jian Sun.
\newblock Repvgg: Making vgg-style convnets great again, 2021.

\bibitem{DBLP:journals/ijcv/FanelliDGFG13}
Gabriele Fanelli, Matthias Dantone, Juergen Gall, Andrea Fossati, and Luc~Van Gool.
\newblock Random forests for real time 3d face analysis.
\newblock {\em Int. J. Comput. Vis.}, 101(3):437--458, 2013.

\bibitem{3ddfa_cleardusk}
Jianzhu Guo, Xiangyu Zhu, and Zhen Lei.
\newblock 3ddfa.
\newblock \url{https://github.com/cleardusk/3DDFA}, 2018.

\bibitem{guo2020towards}
Jianzhu Guo, Xiangyu Zhu, Yang Yang, Fan Yang, Zhen Lei, and Stan~Z Li.
\newblock Towards fast, accurate and stable 3d dense face alignment.
\newblock In {\em Proceedings of the European Conference on Computer Vision (ECCV)}, 2020.

\bibitem{2020NumPy-Array}
Charles~R. Harris, K.~Jarrod Millman, Stéfan~J van~der Walt, Ralf Gommers, Pauli Virtanen, David Cournapeau, Eric Wieser, Julian Taylor, Sebastian Berg, Nathaniel~J. Smith, Robert Kern, Matti Picus, Stephan Hoyer, Marten~H. van Kerkwijk, Matthew Brett, Allan Haldane, Jaime Fernández~del Río, Mark Wiebe, Pearu Peterson, Pierre Gérard-Marchant, Kevin Sheppard, Tyler Reddy, Warren Weckesser, Hameer Abbasi, Christoph Gohlke, and Travis~E. Oliphant.
\newblock Array programming with {NumPy}.
\newblock {\em Nature}, 585:357–362, 2020.

\bibitem{Hempel_2022}
Thorsten Hempel, Ahmed~A. Abdelrahman, and Ayoub Al-Hamadi.
\newblock 6d rotation representation for unconstrained head pose estimation.
\newblock In {\em 2022 IEEE International Conference on Image Processing (ICIP)}. IEEE, October 2022.

\bibitem{hempel2023robust}
Thorsten Hempel, Ahmed~A. Abdelrahman, and Ayoub Al-Hamadi.
\newblock Towards robust and unconstrained full range of rotation head pose estimation, 2023.

\bibitem{article3}
Berthold Horn, Hugh Hilden, and Shahriar Negahdaripour.
\newblock Closed-form solution of absolute orientation using orthonormal matrices.
\newblock {\em Journal of the Optical Society of America A}, 5:1127--1135, 07 1988.

\bibitem{hu2024mathematical}
Huei-Chung Hu, Xuyang Wu, Yuan Wang, Yi~Fang, and Hsin-Tai Wu.
\newblock Mathematical foundation and corrections for full range head pose estimation, 2024.

\bibitem{Joo_2017_TPAMI}
Hanbyul Joo, Tomas Simon, Xulong Li, Hao Liu, Lei Tan, Lin Gui, Sean Banerjee, Timothy~Scott Godisart, Bart Nabbe, Iain Matthews, Takeo Kanade, Shohei Nobuhara, and Yaser Sheikh.
\newblock Panoptic studio: A massively multiview system for social interaction capture.
\newblock {\em IEEE Transactions on Pattern Analysis and Machine Intelligence}, 2017.

\bibitem{dlib09}
Davis~E. King.
\newblock Dlib-ml: A machine learning toolkit.
\newblock {\em Journal of Machine Learning Research}, 10:1755--1758, 2009.

\bibitem{6130513}
Martin Köstinger, Paul Wohlhart, Peter~M. Roth, and Horst Bischof.
\newblock Annotated facial landmarks in the wild: A large-scale, real-world database for facial landmark localization.
\newblock In {\em 2011 IEEE International Conference on Computer Vision Workshops (ICCV Workshops)}, pages 2144--2151, 2011.

\bibitem{Messer1999XM2VTSDBTE}
Kieron Messer, Jiri Matas, Josef Kittler, Juergen Luettin, and Gilbert Ma{\^i}tre.
\newblock Xm2vtsdb: The extended m2vts database.
\newblock 1999.

\bibitem{hpe_article}
Erik Murphy-Chutorian and Mohan Trivedi.
\newblock Head pose estimation in computer vision: A survey.
\newblock {\em IEEE transactions on pattern analysis and machine intelligence}, 31:607--26, 05 2009.

\bibitem{arctan}
Numpy.
\newblock numpy.arctan2.

\bibitem{Ruiz_2018_CVPR_Workshops}
Nataniel Ruiz, Eunji Chong, and James~M. Rehg.
\newblock Fine-grained head pose estimation without keypoints.
\newblock In {\em The IEEE Conference on Computer Vision and Pattern Recognition (CVPR) Workshops}, June 2018.

\bibitem{DBLP:conf/iccvw/SagonasTZP13}
Christos Sagonas, Georgios Tzimiropoulos, Stefanos Zafeiriou, and Maja Pantic.
\newblock 300 faces in-the-wild challenge: The first facial landmark localization challenge.
\newblock In {\em 2013 {IEEE} International Conference on Computer Vision Workshops, {ICCV} Workshops 2013, Sydney, Australia, December 1-8, 2013}, pages 397--403. {IEEE} Computer Society, 2013.

\bibitem{slabaugh1999computing}
Gregory~G Slabaugh.
\newblock Computing euler angles from a rotation matrix.
\newblock 1999.

\bibitem{2020SciPy-NMeth}
Pauli Virtanen, Ralf Gommers, Travis~E. Oliphant, Matt Haberland, Tyler Reddy, David Cournapeau, Evgeni Burovski, Pearu Peterson, Warren Weckesser, Jonathan Bright, St{\'e}fan~J. {van der Walt}, Matthew Brett, Joshua Wilson, K.~Jarrod Millman, Nikolay Mayorov, Andrew R.~J. Nelson, Eric Jones, Robert Kern, Eric Larson, C~J Carey, {\.I}lhan Polat, Yu~Feng, Eric~W. Moore, Jake {VanderPlas}, Denis Laxalde, Josef Perktold, Robert Cimrman, Ian Henriksen, E.~A. Quintero, Charles~R. Harris, Anne~M. Archibald, Ant{\^o}nio~H. Ribeiro, Fabian Pedregosa, Paul {van Mulbregt}, and {SciPy 1.0 Contributors}.
\newblock {{SciPy} 1.0: Fundamental Algorithms for Scientific Computing in Python}.
\newblock {\em Nature Methods}, 17:261--272, 2020.

\bibitem{bfm}
Yue Wu, Juheng Yang, Zhipeng Fan, and Jade Yu.
\newblock 3d-human-face-reconstruction-with-3dmm-face-model-from-rgb-image.

\bibitem{DBLP:conf/iccvw/ZhouFCJY13}
Erjin Zhou, Haoqiang Fan, Zhimin Cao, Yuning Jiang, and Qi~Yin.
\newblock Extensive facial landmark localization with coarse-to-fine convolutional network cascade.
\newblock In {\em 2013 {IEEE} International Conference on Computer Vision Workshops, {ICCV} Workshops 2013, Sydney, Australia, December 1-8, 2013}, pages 386--391. {IEEE} Computer Society, 2013.

\bibitem{zhou2020whenet}
Yijun Zhou and James Gregson.
\newblock Whenet: Real-time fine-grained estimation for wide range head pose, 2020.

\bibitem{DBLP:conf/cvpr/ZhuR12}
Xiangxin Zhu and Deva Ramanan.
\newblock Face detection, pose estimation, and landmark localization in the wild.
\newblock In {\em 2012 {IEEE} Conference on Computer Vision and Pattern Recognition, Providence, RI, USA, June 16-21, 2012}, pages 2879--2886. {IEEE} Computer Society, 2012.

\bibitem{DBLP:conf/cvpr/ZhuLLSL16}
Xiangyu Zhu, Zhen Lei, Xiaoming Liu, Hailin Shi, and Stan~Z. Li.
\newblock Face alignment across large poses: {A} 3d solution.
\newblock In {\em 2016 {IEEE} Conference on Computer Vision and Pattern Recognition, {CVPR} 2016, Las Vegas, NV, USA, June 27-30, 2016}, pages 146--155. {IEEE} Computer Society, 2016.

\bibitem{300w_lp_url}
Xiangyu Zhu, Zhen Lei, Xiaoming Liu, Hailin Shi, and Stan~Z. Li.
\newblock Face alignment across large poses: {A} 3d solution.
\newblock In {\em {CVPR}}, pages 146--155. {IEEE} Computer Society, 2016.

\bibitem{zhu2017face}
Xiangyu Zhu, Xiaoming Liu, Zhen Lei, and Stan~Z Li.
\newblock Face alignment in full pose range: A 3d total solution.
\newblock {\em IEEE transactions on pattern analysis and machine intelligence}, 2017.

\end{thebibliography}

\newpage

\appendix

\section{Validation of HopeNet's draw{\_}axis() }
\label{appendix:300wlp_3line_drawing}
\begin{figure} [H]
\centering
\includegraphics[width=0.9\textwidth]{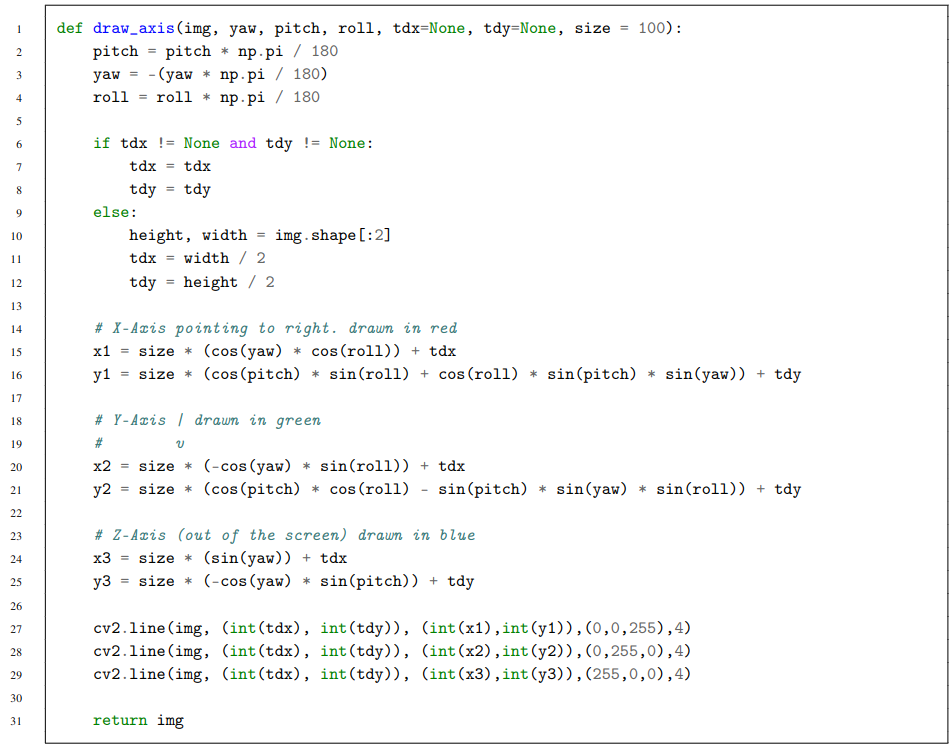}
\caption{HopeNet's \textit{draw{\_}axis} code}
\label{fig:draw_axis}
\end{figure}
While HopeNet's \textit{draw{\_}axis}() (see Figuer \ref{fig:draw_axis}) is popular and seems to depict three lines well, we still need to prove its drawing correct. Due to our earlier discussion in Section \ref{our_300wlp_drawing_section}, we have successfully derived Formula (\ref{formula:30}). Let's expand the formula and substitute $R_{W}$ by Formula (\ref{eq:16}) and $T_{W}$ into $R^{draw}_{W}$. Then we get the following formula:
\begin{equation}
\label{formula:44}
R^{draw}_{W} = \begin{bsmallmatrix}
    cos(y)cos(r) & -cos(y)sin(r) & -sin(y) \\
    cos(p)sin(r)-sin(p)sin(y)cos(r) & cos(p)cos(r)+sin(p)sin(y)sin(r) & -sin(p)cos(y) \\
    sin(p)sin(r)+cos(p)sin(y)cos(r) & sin(p)cos(r)-cos(p)sin(y)sin(r) & cos(p)cos(y) 
\end{bsmallmatrix}. 
\end{equation}
To project the three columns of $R^{draw}_{W}(p, y, r)$ onto the (image's) XY-plane, we can reduce the matrix to the following:
\begin{equation}
\label{formula:45}
R^{draw}_{img}(p, y, r) = \begin{bsmallmatrix}
    cos(y)cos(r) & -cos(y)sin(r) & -sin(y) \\
    cos(p)sin(r)-sin(p)sin(y)cos(r) & cos(p)cos(r)+sin(p)sin(y)sin(r) & -sin(p)cos(y)
\end{bsmallmatrix},
\end{equation}
and the XY-projections for all 3 unit axis vectors are as follows:
\begin{equation}
\label{formula:46}
\begin{split}
\bm{x} &= \begin{bsmallmatrix}
    cos(y)cos(r) \\
    cos(p)sin(r)-sin(p)sin(y)cos(r)
\end{bsmallmatrix}, 
\bm{y} = \begin{bsmallmatrix}
    -cos(y)sin(r) \\
    cos(p)cos(r)+sin(p)sin(y)sin(r) 
\end{bsmallmatrix},
\bm{z} = \begin{bsmallmatrix}
    -sin(y) \\
    -sin(p)cos(y) 
\end{bsmallmatrix}
\end{split}
\end{equation}
To accommodate Line 3's turning yaw to -yaw in Figure \ref{fig:draw_axis}, we replace $cos(y) = cos(-y)$ and $sin(y) = -sin(-y)$ into Formula \ref{formula:46} to get Formula (\ref{formula:47}). Next, follow Line 3 set $\tilde{y} = -y$ in Formula (\ref{formula:47}) and we get Formula (\ref{formula:48} )
\begin{equation}
\label{formula:47}
\begin{split}
\bm{x'} &= \begin{bsmallmatrix}
    cos(-y)cos(r) \\
    cos(p)sin(r)+sin(p)sin(-y)cos(r)
\end{bsmallmatrix},
\bm{y'} = \begin{bsmallmatrix}
    -cos(-y)sin(r) \\
    cos(p)cos(r)-sin(p)sin(-y)sin(r) 
\end{bsmallmatrix},
\bm{z'} = \begin{bsmallmatrix}
    sin(-y) \\
    -sin(p)cos(-y)
\end{bsmallmatrix};
\end{split}
\end{equation}
\begin{equation}
\label{formula:48}
\begin{split}
    \begin{bsmallmatrix}
        \textrm{x1} \\ 
        \textrm{y1}
    \end{bsmallmatrix} &= \begin{bsmallmatrix}
    cos(\tilde{y})cos(r) \\
    cos(p)sin(r)+sin(p)sin(\tilde{y})cos(r)
    \end{bsmallmatrix},
    \begin{bsmallmatrix}
        \textrm{x2} \\ 
        \textrm{y2}
    \end{bsmallmatrix}  = \begin{bsmallmatrix}
    -cos(\tilde{y})sin(r) \\
    cos(p)cos(r)-sin(p)sin(\tilde{y})sin(r) 
    \end{bsmallmatrix},
    \begin{bsmallmatrix}
        \textrm{x3} \\ 
        \textrm{y3}
    \end{bsmallmatrix}  = \begin{bsmallmatrix}
    sin(\tilde{y}) \\
    -sin(p)cos(\tilde{y}) 
    \end{bsmallmatrix}.
\end{split}
\end{equation}
Since Formula (\ref{formula:48}) completely matches with Lines 15-25 of Figure \ref{fig:draw_axis}, we proved \textbf{\textit{draw{\_}axis}() equals to our proposed drawing routine (Formula (\ref{formula:46}))}. 
\begin{table*}
    \begin{minipage}{0.4\linewidth}
        \label{table:300wlp_labels}
        \centering
\begin{tabular}{|c|c|c|c|} 
 \hline
 pose & left & middle & right \\ [0.1ex] 
 \hline
 pitch & 6.208 & -17.325 & -7.601 \\ [0.1ex] 
 \hline 
 yaw & 5.876 & -49.589 & -54.009 \\ [0.1ex] 
 \hline
 roll & -1.694 & 11.423 & 4.450 \\ [0.1ex] 
 \hline
\end{tabular}
     \caption{Head pose labels for Figure \ref{fig:3drawings}}
    \end{minipage}
    \begin{minipage}{0.5\linewidth}
    \begin{figure}[H]
        \centering
            \includegraphics[width=.8\textwidth]{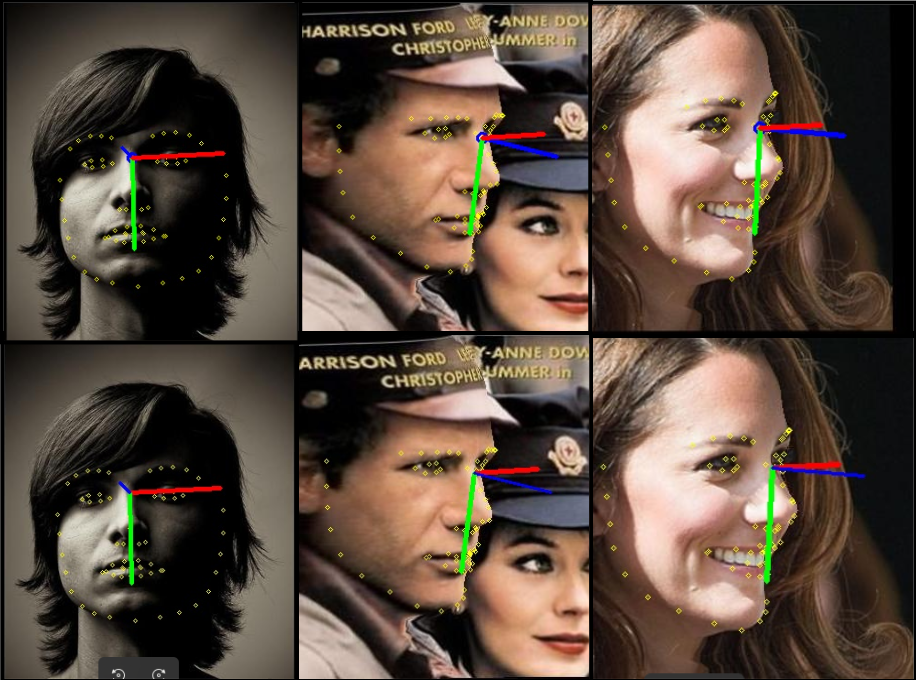}
    \caption{Draw Table 1's head poses with our drawing routine (top) and \textbf{draw{\_}axis} (down). Images are from 300W-LP dataset.}
    \label{fig:3drawings}
    \end{figure}
    \end{minipage}
\end{table*}

\section{Euler Angle Extraction for Rotations in 300W-LP's Rotation System}
\label{appendix:300wlp_euler_extraction}
We follow Slabaugh, G. G's \cite{slabaugh1999computing} Euler angles' computation except that the range $(-\pi, \pi]$ for the yaw, pitch, and roll values are selected. Although for the pitch and roll, we mainly focus on $[-\pi/2, \pi/2]$ because most head poses occur in this range, including the upside-down human head pose as a result of image augmentations, we cannot eliminate the possibility that true full-range may be useful in some occasions. We also use NumPy \footnote{\url{https://numpy.org}} for trigonometric calculation in code implementations.  

For a given rotation $R=(R_{i,j})$ from 300W-LP where $R_{i,j}$ is the $(i,j)$-the entry in R, we can set up the matrix equation (\ref{eq:17}) to solve for yaw, pitch, and roll.  
\begin{equation}
\label{eq:17}
\begin{split}
(R_{i,j}) &= \begin{bsmallmatrix}
    cos(y)cos(r) & cos(y)sin(r) & -sin(y) \\
    -cos(p)sin(r)+sin(p)sin(y)cos(r) & cos(p)cos(r)+sin(p)sin(y)sin(r) & sin(p)cos(y) \\
    sin(p)sin(r)+cos(p)sin(y)cos(r) & -sin(p)cos(r)+cos(p)sin(y)sin(r) & cos(p)cos(y) 
\end{bsmallmatrix} 
\end{split}
\end{equation}

\bigbreak
Equation $R_{0,2} = -\bm{sin}(y)$ gives yaw's first solution $y_{1} = \bm{arcsin}(-R_{0,2}) \in [-\pi/2, \pi/2]$ due to Numpy's $\bm{arcsin}$ function call.  To restrict both yaws in $(-\pi, \pi]$, we derive the second formula of yaw in Formula (\ref{eq:18}):

\begin{equation}
\label{eq:18}
  y_{2} =
    \begin{cases}
      \pi - y_{1} &  if \verb|  | y_{1} \ge 0  \\
      -\pi - y_{1} & if \verb|  |  y_{1} < 0  
    \end{cases}       
\end{equation}
\noindent
Gimbal locks happen when $\bm{cos}(y) = 0$, i.e, $y_1 = y_2 \text{ and } |y_i| = \pi/2$.  We will discuss it later. 

For the non-Gimbal-lock cases, equivalently $\bm{cos}(y) \neq 0$, we consider the two equations shown in Formula (\ref{eq:19}).
\begin{equation}
\label{eq:19}
    \begin{cases}
      R_{1,2} &= \bm{sin}(p)\bm{cos}(y) \\
      R_{2,2} &= \bm{cos}(p)\bm{cos}(y)
    \end{cases}       
\end{equation}
Adapting to Numpy's \cite{2020NumPy-Array} $\bm{arctan2}$ \cite{arctan} makes the solution unique. When applying $\bm{arctan2}$, we must be careful when $\bm{cos}(y) < 0$, as mentioned in \cite{slabaugh1999computing}.  It will result in the sign change for both terms, $\bm{sin}(p)$ and $\bm{cos}(p)$, and lead to the wrong angle prediction, $p \pm \pi$, instead of $p$.  To avoid such a pitfall, we set pitch $p = \textbf{arctan2}(R_{1,2}/\bm{cos}(y), R_{2,2}/\bm{cos}(y))$, which results in $p$ in $[-\pi, \pi]$. From Numpy's $\bm{arctan2}$’s document \cite{arctan}, $p$ can be further restricted into the range $(-\pi, \pi)$ due to both numerator and denominator can’t be $\pm \inf$ and $\pm 0$ at the same time.  Now, we can formulate pitch as follows:
\begin{equation}
\label{eq:20}
    \begin{cases}
      p_{1} &= \bm{arctan2}(R_{1,2}/\bm{cos}(y_{1}), R_{2,2}/\bm{cos}(y_{1}))\\
      p_{2} &= \bm{arctan2}(R_{1,2}/\bm{cos}(y_{2}), R_{2,2}/\bm{cos}(y_{2}))
    \end{cases}       
\end{equation}
\noindent
Formula (\ref{eq:18}) implies that $\bm{cos}(y_2) = -\bm{cos}(y_1)$. Now substitute it back to Formula (\ref{eq:20}), then we get $\bm{sin}(p_2) = -\bm{sin}(p_1)$ and $\bm{cos}(p_2) = -\bm{cos}(p_1)$.  Hence $p_{1}$ and $p_{2}$ differs by $\pi$.  To get both of them inside $(-\pi, \pi]$,  we get a simplified formula for pitch, shown in Formula (\ref{align:21-22}):  
\begin{align} 
  \label{align:21-22}
    p_{1} &= \bm{arctan2}(R_{1,2}/\bm{cos}(y_{1}), R_{2,2}/\bm{cos}(y_{1}))\\
    p_{2} &= \begin{cases}
               p_{1} - \pi &\textbf{  if } p_{1} \ge 0 \\
               p_{1} + \pi &\textbf{  if } p_{1} < 0 
            \end{cases}   
\end{align}
\noindent
Similarly, we can derive the formula for roll, as shown in Formula (\ref{align:23-24}):
\begin{align} 
  \label{align:23-24}
    r_{1} &= \bm{arctan2}(R_{0,1}/\bm{cos}(y_{1}), 
              R_{0,0}/\bm{cos}(y_{1}))\\
    r_{2} &= \begin{cases}
               r_{1} - \pi &\textbf{  if } r_{1} \ge 0 \\
               r_{1} + \pi &\textbf{  if } r_{1} < 0 
            \end{cases}   
\end{align}
Gimbal locks occur at $y=\pm \pi/2$. If $y = \pi/2$, $\bm{sin}(y) = 1$ and $\bm{cos}(y)=0$ reduce the matrix (\ref{eq:17}) into Formula (\ref{eq:25}). In which case, we also find that $R_{1,0} = \bm{sin}(p - r)$ and $R_{1,1} = \bm{cos}(p - r)$.  So we have $p-r=\bm{arctan2}(R_{1,0}, R_{1,1})$. While there are infinitely many solutions for $p$ and $r$, to prevent $|p| > \pi/2$, we decide to set $p =  \bm{arctan2}(R_{1,0}, R_{1,1}) / 2 $ and $r = -p$ to force them to fall within $[-\pi/2, \pi/2]$.  
\begin{equation}
\label{eq:25}
\begin{split}
R(p, \pi/2, r) &= \begin{bsmallmatrix}
    0 & 0 & -1 \\
    -cos(p)sin(r)+sin(p)cos(r) & cos(p)cos(r)+sin(p)sin(r) & 0 \\
    sin(p)sin(r)+cos(p)cos(r) & -sin(p)cos(r)+cos(p)sin(r) & 0
    \end{bsmallmatrix} \\ \\
    p &=  \bm{arctan2}(R_{1,0}, R_{1,1}) / 2,  \textbf{   } r = -p
\end{split}
\end{equation}
The other Gimbal lock case is when $y = -\pi/2$. Then the matrix in Formula (\ref{eq:17}) becomes Formula (\ref{eq:26}). Similarly, we can derive $R_{1,0}=-\bm{sin}(p+r)$, $R_{1,1} = \bm{cos}(p+r)$ and $p+r = \bm{arctan2}(-R_{1,0}, R_{1,1})$.  Again, to make both $p$ and $r$ in $[-\pi/2, \pi/2]$, we set $p=r=\bm{arctan2}(-R_{1,0}, R_{1,1}) / 2$.
\begin{equation}
\label{eq:26}
\begin{split}
R(p, -\pi/2, r) &= \begin{bsmallmatrix}
    0 & 0 & 1 \\
    -cos(p)sin(r)-sin(p)cos(r) & cos(p)cos(r)-sin(p)sin(r) & 0 \\
    sin(p)sin(r)-cos(p)cos(r) & -sin(p)cos(r)-cos(p)sin(r) & 0
    \end{bsmallmatrix}, \\ \\
    p =  r &= \bm{arctan2}(-R_{1,0}, R_{1,1}) / 2 \\
\end{split}
\end{equation}
Figure \ref{fig:300wlp_extraction} presents the full Python code for computing the closed-form yaw, roll, and pitch for the 300W-LP. 
\begin{figure} [H]
\centering
\includegraphics[width=0.85\textwidth]{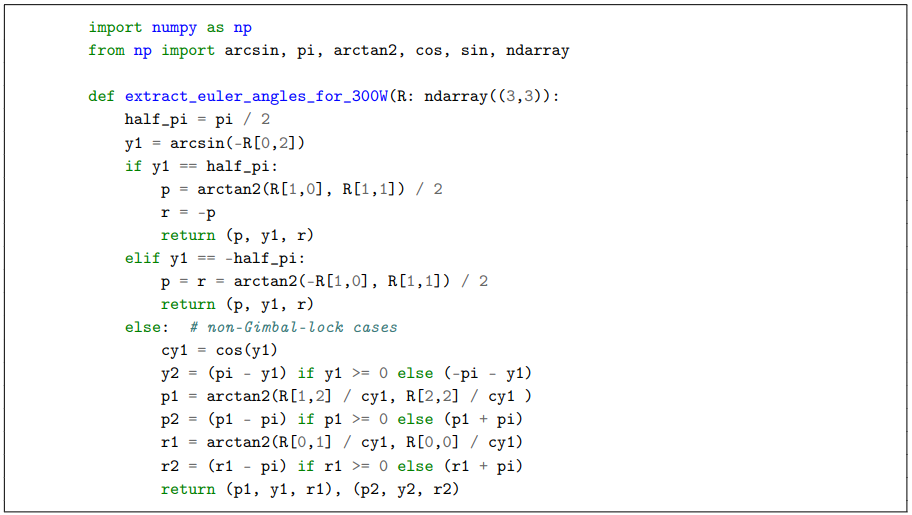}
\caption{Euler angle extraction for rotations in 300W-LP's rotation system.}
\label{fig:300wlp_extraction}
\end{figure}

\section{Proposed 3D Flipping and Rotation Augmentations}
\label{appendix: 3dflip_rot}
\begin{defn} \label{aug_type_def} The 2D images are assumed to be on the $XY$-plane of the 300W-LP's coordinate system as in Figure \ref{fig:300w-lp-system2}. 2D images' \textbf{horizontal \& vertical flipping} are the flipping across the image's vertical Y-axis and horizontal X-axis, respectively. On the other hand, to align with OpenCV's built-in affine transformation, we define a \textbf{image's rotation with the angle $\bm{\phi}$} to be the rotation rotating the image an angle, $\phi$, counter-clockwise. 
\end{defn} 

\begin{defn} \label{L_def}
Define $\bm{L_{\theta}}$ to be the line passing through the origin and the angle from the horizontal axis and itself (on the vertical plane, denoted by $\bm{P_{v}}$, intersecting with the horizontal and vertical axes) is $\theta$. For example, the line $L_{\theta}$ of 300W-LP's coordinate system represents one constructed from rotating the X-axis ($y = 0$) by the angle $\theta$ counter-clockwise on the XY-plane.  Because of 300W-LP's left-handed nature for the XY-plane's rotations, such rotations are clockwise. So $\theta > 0$ is equivalent to the roll rotation $r = - \theta$ in 300W-LP's rotation system. However, Wikipedia's rotation system's right-handed nature guarantees its roll rotation $r = \theta$ will remain the same sign as $L_{\theta}$'s $\theta$.
\end{defn}

\begin{defn} \label{gen_flip_def}
Next, we define the \textbf{2D flipping about $\bm{L_{\theta}}$} to be the flipping across $L_{\theta}$ on the vertical plane $\bm{P_{v}}$ defined in Definition \ref{L_def}. 
\end{defn}

We will derive formulas under 300W-LP's coordinate system and conventions. Similar derivations can be carried out easily for other coordinate systems.
\begin{theorem} \label{thm:2d_rotation}
Fix a 3D rotation system. Suppose an image contains a human head, and the head's intrinsic rotation $R$ in the rotation system is given. Prove that the 3D rotation, denoted by $R_{rotate{\_}img}(\phi)$, associated with rotating this image by an angle $\phi$ is the same to apply the extrinsic rotation, denoted as $R_{extr}(\phi)$, of rotating the axis (perpendicular to the image) by the angle $\phi$ on $R$. Hence, we have 
\begin{equation}
  \label{formula:58}
    R_{rotate{\_}img}(\phi) = R_{extr}(\phi)  \times R      
\end{equation}
\end{theorem}

\begin{proof}
Without loss of generality, let's fix to 300W-LP's rotation system and $R \ne I_{3}$.  Due to 300W-LP's clockwise-roll nature, rotating the image by $\phi$ counter-clockwise results in the roll's elemental rotation $R^{left}_{Z}(-\phi)$ defined in Formula \ref{eq:16}. Moreover, this rotation is extrinsic because the image's rotation can only happen on 300W-LP's extrinsic $xy$-plane. Then Theorem \ref{thm:diff_rots2} implies that $R_{rot{\_}img}(\phi)$ satisfies Formula (\ref{eq:39}).
\begin{equation}
  \label{eq:39}
  \begin{split}
    &R_{extr}(\phi) = R^{left}_{Z}(-\phi), \\
    &R_{rot{\_}img}(\phi) = R_{extr}(\phi) \times R = R^{left}_{Z}(-\phi) \times R \\ \implies 
    &R_{rot{\_}img}(\phi)
    =  \begin{bmatrix}
        cos(\phi) & -sin(\phi) & 0\\
        sin(\phi) & cos(\phi) & 0\\
        0 & 0 & 1
       \end{bmatrix} \times R 
  \end{split}
\end{equation}
\end{proof}
\begin{theorem} \label{thm:flip}
Let's fix 300W-LP's rotation system and restrict $0\degree \leq \theta \leq 90\degree$. Suppose an image contains a human head, and the head's intrinsic rotation $R$ in the fixed rotation system is given. Prove that the 3D rotation $R_{flip}(\theta)$, corresponding to the 2D flipping across $L_{\theta}$ and the given head's rotation $R$, can be expressed as follows:
\begin{equation}
  \label{formula:52}
    R_{flip}(\theta) = \begin{bmatrix}
        cos(2\theta) & sin(2\theta) & 0\\
        sin(2\theta) & -cos(2\theta) & 0\\
        0 & 0 & 1
       \end{bmatrix}  \times R \times \begin{bmatrix}
        -1 & 0 & 0\\
        0 & 1 & 0\\
        0 & 0 & 1
       \end{bmatrix}     
\end{equation}
\end{theorem}
\begin{proof}
Without loss of generality, let $R \ne I_{3}$. It's trivial that the 2D flipping across $L_{\theta}$ will flip from $\bm{e_{x}} \coloneqq (1, 0, 0)$ to $(cos(2\theta), sin(2\theta), 0)$. But the 2D flipping across $L_{\theta}$ flips $\bm{e_{y}} \coloneqq (0, 1, 0)$ to $(sin(2\theta), -cos(2\theta), 0)$ requires some work. 
So the matrix of the extrinsic 3D flipping is $\begin{bmatrix}
        cos(2\theta) & sin(2\theta) & 0\\
        sin(2\theta) & -cos(2\theta) & 0\\
        0 & 0 & 1
       \end{bmatrix}$. 
The flipping also causes the intrinsic X-axis to change to the opposite direction, which results in the intrinsic flipping across the X-axis, i.e., 
$\begin{bmatrix}
-1 & 0 & 0 \\
0 & 1 & 0 \\
0 & 0 & 1
\end{bmatrix}$.  Therefore, the 2D flipping across $L_{\theta}$ relates to the 3D rotation, which applies an extrinsic $L_{\theta}$-flipping and an intrinsic $X$-flipping. By Theorem \ref{thm:diff_rots2}, we have $R_{flip}(\theta) =  (extrinsic{\_}L_{\theta}{\_}flipping) \times R  \times (intrinsic{\_}X{\_}flipping)$, and thus proved Formula (\ref{formula:52}).
\end{proof}


\section{Euler angle conversion from pitch-yaw-roll to roll-pitch-yaw axis-sequences }
\label{appendx:pyr_2_rpy}
Let's start with 300W-LP's Euler angle definitions $R^{left}_{X}(p), R^{left}_{Y}(y), R^{left}_{Z}(r)$ and $R_{W}(y, p, r)$ in Formula (\ref{eq:16}). Then Blender's image generation with given pitch and yaw values as input follows 300W-LP's pitch-yaw rotations:
\begin{equation}
  \label{formula:panohead}
  \begin{split}
    R_{B}(p', y') &\coloneqq R^{left}_{X}(p') \times R^{left}_{Y}(y') 
    = \begin{pmatrix}
    cos(y') & 0 & -sin(y') \\
    sin(p')sin(y') & cos(p') & sin(p')cos(y') \\
    cos(p')sin(y') & -sin(p') & cos(p')cos(y') 
    \end{pmatrix}
  \end{split}
\end{equation}

To overcome Blender's pure pitch-yaw head pose images,  we first apply the 3D rotation with Formula (\ref{formula:1}) to get the desired non-zero roll rotation.  Next we need to solve there exists a numerically stable solution for $y', p', r'$ in the following equation 
\begin{equation}
\label{formula:3wlp_panohead_relation}
\begin{split}
&R_{W}(y, p, r) = R_{B}(p', y', r') ,\quad where \\
R_{B}(p', y', r') &\coloneqq R^{left}_{Z}(r') \times R_{B}(p', y')
= R^{left}_{Z}(r') \times R^{left}_{Y}(y) \times R^{left}_{Z}(r).
\end{split}
\end{equation}

Or equivalently it suffices to solve for $y', p', r'$ in the following equation with given $p, y, r$ values:
\begin{equation}
  \label{formula:3wlp_2_panohead}
  \begin{split}
    &R_{W}(y, p, r) =\\
    &\begin{pmatrix}
    sin(p')sin(r')sin(y')+cos(r')cos(y') & sin(r')cos(p')  &  sin(p')sin(r')cos(y')-sin(y')cos(r')\\
    sin(p')sin(y')cos(r')-sin(r')cos(y') & cos(p')cos(r') & sin(p')cos(r')cos(y') + sin(r')sin(y') \\
    sin(y')cos(p') & -sin(p') & cos(p')cos(y') 
    \end{pmatrix}.\\
  \end{split}
\end{equation}
We use Appendix \ref{appendix:300wlp_euler_extraction}'s methodology to derive to our Blender and Panohead's \cite{an2023panohead} Euler angle extraction solution, Fig. \ref{listing:1}.  It is worthwhile to mention that the Frobenius-norm errors between $R_{W}(p,y,r)$ and $R_{B}(p', y', r')$ remains in $10^{-11}$ range from our experiments.  So our extraction method, Fig. \ref{listing:1}, provides a numerically stable Euler angle solution for our synthetic head pose dataset generation with Blender and Panohead's help.  


\begin{figure}[H]
\centering
\includegraphics[width=0.9\textwidth]{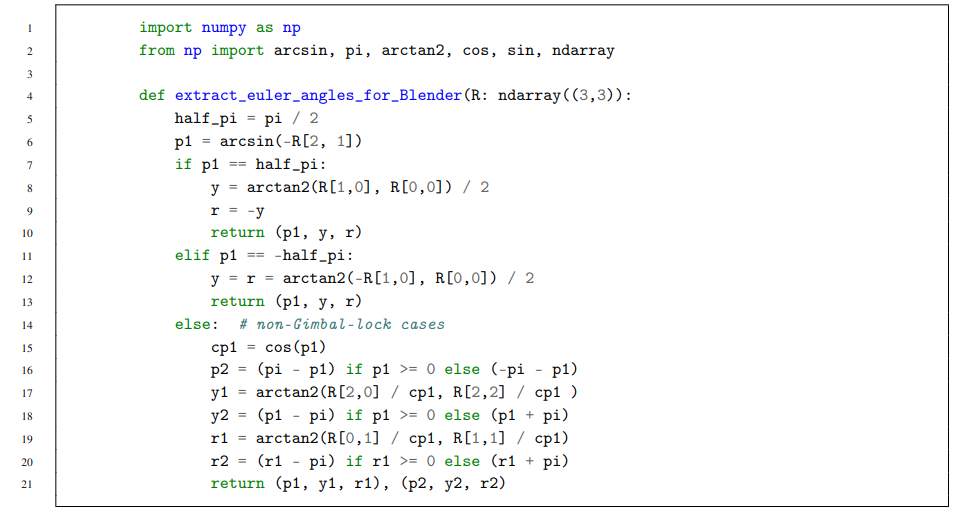}
\caption{Complete yaw, roll, pitch extraction for Blender's roll-pitch-yaw image generation.}
\label{listing:1}
\end{figure}



\end{document}